\documentclass{article}
\usepackage{arxiv}

\usepackage[T1]{fontenc}
\usepackage[utf8]{inputenc}
\usepackage{amsmath,amssymb,amsthm}
\usepackage{booktabs}
\usepackage{graphicx}
\usepackage[hyphens]{url}
\usepackage{natbib}
\usepackage{hyperref}

\theoremstyle{plain}
\newtheorem{lemma}{Lemma}
\newtheorem{proposition}{Proposition}
\newtheorem{corollary}{Corollary}
\theoremstyle{definition}
\newtheorem{assumption}{Assumption}
\theoremstyle{remark}
\newtheorem{remark}{Remark}

\theoremstyle{plain}
\newtheorem{lemmaS}{Lemma}
\newtheorem{propositionS}{Proposition}
\newtheorem{corollaryS}{Corollary}
\theoremstyle{definition}
\newtheorem{assumptionS}{Assumption}

\newcommand{\R}{\mathbb{R}}
\newcommand{\E}{\mathbb{E}}
\newcommand{\Wq}{W_Q}
\newcommand{\Wk}{W_K}
\newcommand{\mQ}{m_Q}
\newcommand{\mK}{m_K}
\newcommand{\Rot}{\mathrm{Rot}}

\title{Relative Positions Generalize, Absolute Positions Memorize:\\
An Implicit-Bias Account of Length Generalization in Attention}

\author{
  Subham Singh\textsuperscript{\rm 1}, Ashutosh Mishra\textsuperscript{\rm 2}, Subha Raut\textsuperscript{\rm 1} \\[4pt]
  \normalfont\textsuperscript{\rm 1}Department of Mathematics and Statistics, Mississippi State University \\
  \normalfont\textsuperscript{\rm 2}Independent Researcher \\
  \normalfont ss4707@msstate.edu, ashutoshm1771@gmail.com, sr2120@msstate.edu
}
\date{}

\begin{document}
\maketitle

\begin{abstract}
Transformers with relative positional encodings often extrapolate to sequences
longer than those seen during training, whereas transformers with learned
absolute encodings typically do not. This is a robust empirical regularity, and
the explanations offered for it so far are chiefly about \emph{expressivity},
that is, about whether a length-generalizing solution exists. We give an
\emph{optimization} explanation. On a minimal fixed-offset retrieval task that
isolates positional selection, the gap is governed by the implicit bias of the
trained attention head: among the many solutions that fit short sequences,
which one gradient descent actually selects. We prove that rotary encodings
make the attention logit a function of relative offset alone, an exact
equivariance, so whatever selection rule is learned at training lengths is
reproduced verbatim at every longer length. Learned absolute encodings instead
leave out-of-range positions unconstrained, and the trained head pins to a
fixed absolute position inside the training range. We characterize the learned
rotary rule as a low-rank ``carrier'' kernel aligned with the target offset,
and we derive the resulting graceful accuracy decay as an attention-dilution
law; both predictions are confirmed across seeds and offsets. A
linear-attention control shows the mechanism is specific to softmax: without
normalization, training selects a min-norm interpolant that does not
extrapolate. The phenomenon, the equivariance, and the carrier all transfer
to a multi-layer, multi-head transformer trained on a full-sequence
length-generalization task. The account connects the implicit bias of
attention, implicit bias for extrapolation in recurrent models, and the
learning side of the RASP-L conjecture.
\end{abstract}

\section{Introduction}
Length generalization, the ability to train on short sequences and succeed on
longer ones, is a basic requirement for sequence models and a longstanding
puzzle \citep{anil2022exploring,deletang2023chomsky}. One finding recurs across
tasks and architectures: the \emph{positional encoding} matters, and relative
schemes extrapolate far better than learned absolute encodings
\citep{press2022alibi,kazemnejad2023nope,jelassi2023arithmetic}.
The prevailing explanations are about \emph{expressivity}. The RASP-L
conjecture \citep{zhou2023rasp} and the formal-language framework of
\citet{huang2025formal} characterize whether a length-generalizing solution
\emph{exists}, or whether it is identifiable from the training distribution.
Existence, however, is not selection. Among the many parameter settings that
fit the training lengths, gradient descent picks one, and whether that one
generalizes is an optimization question this literature leaves open. Indeed,
the RASP-L conjecture is itself implicitly a claim about what training
\emph{learns}, supported empirically rather than explained.

We address this optimization question on a task small enough to analyze
exactly and run exhaustively, studying the implicit bias of a single trained
attention head: which solution training selects, and how the positional
encoding shapes that selection.

\paragraph{The phenomenon (Section~\ref{sec:phenomenon}).}
Consider a fixed-offset retrieval task: attend to the token a fixed number of
positions before the query and copy it. A one-layer head with rotary encoding
(RoPE) keeps near-perfect accuracy well beyond the training-length range. The
same head with a learned absolute encoding (APE) collapses to near-chance
accuracy the moment it encounters an unseen length.

\paragraph{Contributions.}
\begin{enumerate}
\item \textbf{The phenomenon at minimal scale.} We isolate the
relative-versus-absolute gap in a one-layer, single-head model on a task with
a single optimal token, the regime in which the attention-SVM
characterization of implicit bias \citep{tarzanagh2023svm} applies
(Sections~\ref{sec:setup}--\ref{sec:phenomenon}).
\item \textbf{Transfer is exact equivariance (Lemma~\ref{lem:equiv}).} The
RoPE logit depends on positions only through their relative offset, so the
selected offset is length-invariant. Empirically the selected offset equals
the target in $100\%$ of seed$\times$length cells.
\item \textbf{The selected rule is a target-aligned carrier
(Lemma~\ref{lem:select}), and accuracy decays by dilution
(Corollary~\ref{cor:decay}).} Trained query and key representations share a
dominant direction, pre-rotated to the target offset, which yields a relative
kernel peaked at the target; softmax over a growing context then forces an
inverse-linear decay of the target's attention weight. Both effects are
confirmed across seeds and offsets $K\in\{2,3\}$. Removing softmax
normalization removes the mechanism: trained linear attention solves the
task in distribution by min-norm interpolation and does not extrapolate
(Section~\ref{sec:disc}).
\item \textbf{Absolute encodings are position-pinned
(Proposition~\ref{prop:ape}).} As a negative control, we show the trained APE
head selects a fixed absolute position inside the training range, so it misses
the target outside that range.
\item \textbf{The mechanism is not a one-layer artifact
(Section~\ref{sec:real}).} In a two-layer, four-head transformer with MLP and
LayerNorm trained on a full-sequence shift task, all three findings
transfer: RoPE generalizes while APE degrades out of range, the RoPE head remains
offset-equivariant out of distribution while the APE head pins, and it
carries a strong target carrier ($0.98$ versus a $0.02$ baseline).
\end{enumerate}

\section{Setup}
\label{sec:setup}
\paragraph{Task.} The vocabulary is $V=\{1,\dots,m\}$, and a length-$L$
sequence $s_0,\dots,s_{L-1}$ consists of i.i.d.\ uniform tokens. The query is
the last position, and the target is the token a fixed offset $K$ earlier:
$y=s_{L-1-K}$. Training lengths are drawn from $[L_{\min},L_{\max}]$;
evaluation uses lengths $L'>L_{\max}$. The task therefore isolates a purely
positional selection problem with a single optimal token.

\paragraph{Model.} We use one layer and a single head, with the query taken
from the last position and no MLP. With token embeddings $x_t=E(s_t)\in\R^d$
and projections $\Wq,\Wk,W_V$, the logit between the query position $n=L-1$
and key position $t$ is $\ell_t=q^\top k_t/\sqrt d$. We compare \textbf{APE},
which adds learned position vectors ($z_t=x_t+p_t$, $q=\Wq z_n$,
$k_t=\Wk z_t$), with \textbf{RoPE}, which applies block-diagonal rotations
$R_t$ built from $2\times2$ blocks $\Rot(t\theta_i)$ on each coordinate pair
$i$ ($q=R_n\Wq x_n$, $k_t=R_t\Wk x_t$). The output is
$o=\sum_t \mathrm{softmax}(\ell)_t v_t$ with prediction $\hat y=Uo$.

\paragraph{Background.} \citet{tarzanagh2023svm} show that, under strictly
decreasing losses with vanishing regularization, the key--query
product converges in direction to a max-margin solution that separates the
optimal token from the rest. We study the \emph{positional structure} of that
selected solution.

\section{The Phenomenon: Relative Generalizes, Absolute Collapses}
\label{sec:phenomenon}
\begin{figure}[t]\centering
\includegraphics[width=\linewidth]{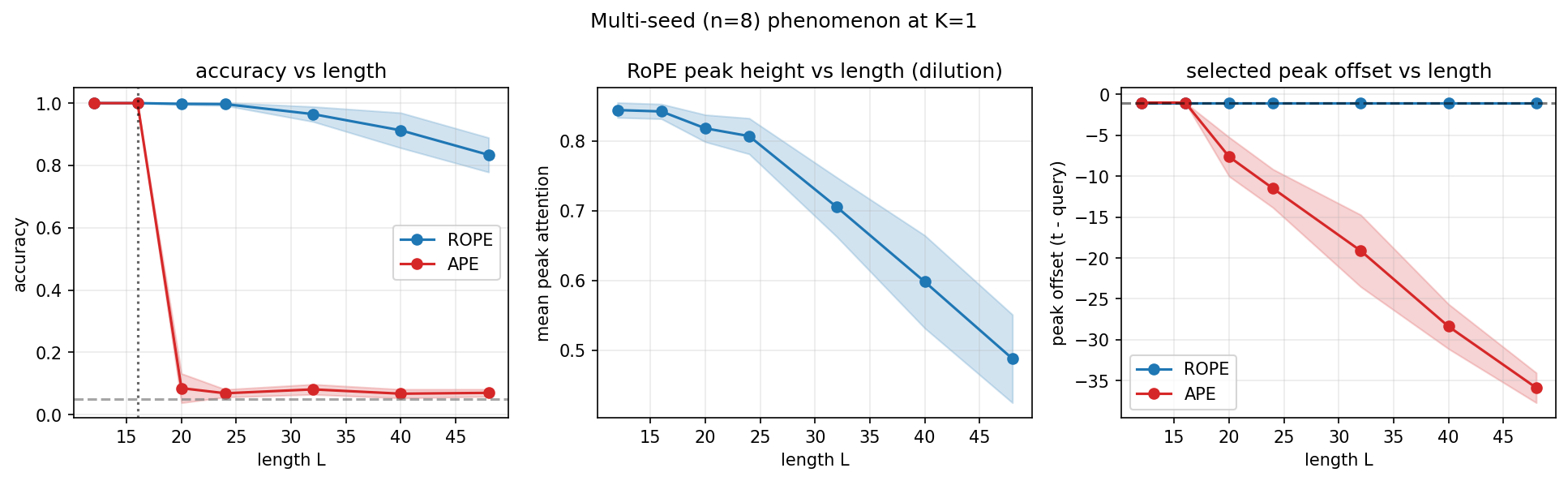}
\caption{Multi-seed ($n{=}8$) results at $K{=}1$. \textbf{Left:} accuracy
versus length; RoPE stays high past the training range (dotted line) while APE
drops to near chance (dashed line). \textbf{Middle:} the RoPE peak attention
height decays gracefully. \textbf{Right:} the selected peak offset is flat at
$-K$ for RoPE and diverges for APE.}
\label{fig:phenomenon}
\end{figure}
Figure~\ref{fig:phenomenon} and Table~\ref{tab:acc} report the separation.
RoPE accuracy moves from $1.00$ to $0.83$ as $L$ grows from $12$ to $48$,
whereas APE is at $1.00$ in range and near $0.07$ (chance is $0.05$) out of
range. The same separation holds at $K=2$ and $K=3$ (Section~\ref{sec:exp}).
The rest of the paper explains the success, its graceful erosion, and the
failure.

\begin{table}[t]\centering
\caption{Accuracy by length ($K{=}1$, mean over $8$ seeds). Lengths
$L\le16$ are in distribution. Chance is $0.05$.}
\label{tab:acc}
\small
\setlength{\tabcolsep}{3.6pt}
\begin{tabular}{lccccccc}
\toprule
$L$ & 12 & 16 & 20 & 24 & 32 & 40 & 48\\
\midrule
RoPE & $1.00$ & $1.00$ & $1.00$ & $1.00$ & $0.96$ & $0.91$ & $0.83$\\
APE & $1.00$ & $1.00$ & $0.09$ & $0.07$ & $0.08$ & $0.07$ & $0.07$\\
\bottomrule
\end{tabular}
\end{table}

\section{An Implicit-Bias Analysis}
\subsection{Rotary Attention Is Offset-Equivariant}
\begin{lemma}[Equivariance]\label{lem:equiv}
For any $\Wq,\Wk$, the RoPE logit between a query at position $n$ and a key at
position $t$ depends on $(n,t)$ only through the relative offset $\delta=t-n$:
$\sqrt d\,\ell_t = x_n^\top \Wq^\top R_{\delta}\Wk\, x_t =: \Phi(\delta;x_n,x_t).$
\end{lemma}
\begin{proof}
$R_t$ is block-diagonal with blocks $\Rot(t\theta_i)$. Since
$\Rot(\alpha)^\top\Rot(\beta)=\Rot(\beta-\alpha)$ blockwise, we have
$R_n^\top R_t=R_{t-n}$, and therefore $\sqrt d\,\ell_t=(R_n\Wq x_n)^\top(R_t\Wk x_t)
=x_n^\top\Wq^\top R_n^\top R_t\Wk x_t=x_n^\top\Wq^\top R_{t-n}\Wk x_t$.
\end{proof}
The relative dependence in Lemma~\ref{lem:equiv} is the design property of
rotary encodings \citep{su2021roformer}; we state it because the
length-invariance it yields (Corollary~\ref{cor:transfer}) grounds the account.
\begin{corollary}[Rule transfer]\label{cor:transfer}
The score profile $\Phi$ carries no dependence on sequence length, so the
offset that maximizes it is the same at every length at which that offset is
available, provided no newly available offset attains a larger value. Under
Assumption~\ref{ass:carrier} below, the expected profile is globally maximized
at $\delta=-K$ (Lemma~\ref{lem:select}), so for the fixed-offset task the
selected offset is $-K$ independent of $L$.
\end{corollary}
Empirically (Fig.~\ref{fig:phenomenon}, right; Section~\ref{sec:exp}) the
selected offset equals $-K$ in $100\%$ of seed$\times$length cells with zero
variance, as expected for a structural identity. The invariance requires no
assumption about the data or training and is what makes length
generalization possible at all.

\subsection{The Selected Solution Is a Target-Aligned Carrier}
\begin{assumption}[Shared carrier]\label{ass:carrier}
There exist a unit vector $u$ and a scalar $\rho>0$ such that
$\Wk x_t=\rho u+\xi_t$ and $\Wq x_n=\rho R_{-K}u+\xi'_n$, where the content
residuals satisfy $\E[\xi_t]=\E[\xi'_n]=0$ and $\xi_t\perp u$.
\end{assumption}
\begin{lemma}[Target-aligned selection]\label{lem:select}
Under Assumption~\ref{ass:carrier}, the content-averaged score profile over
keys $t\ne n$ (offsets $\delta\le-1$) is
\begin{align*}
\phi(\delta)&:=\E\!\left[\sqrt d\,\ell_t\mid \delta\right]
=\mQ^\top R_\delta \mK=\rho^2\,g(\delta+K),\\
g(s)&=\sum_{i=1}^{d/2} p_i\cos(s\theta_i),
\end{align*}
with $\mK=\rho u$, $\mQ=\rho R_{-K}u$, and pair energies
$p_i=u_{2i-1}^2+u_{2i}^2\ge0$, $\sum_i p_i=1$. Then $g(s)\le g(0)=1$, so
$\phi$ is maximized at $\delta=-K$. If, moreover, no nonzero integer $s$ in
the evaluated offset range satisfies $s\theta_i\in2\pi\mathbb{Z}$ for all $i$
with $p_i>0$, the maximizer is unique.
\end{lemma}
\begin{proof}
For $t\ne n$ the tokens $x_n$ and $x_t$ are independent, so
$\E[\sqrt d\,\ell_t]=\E[\Wq x_n]^\top R_\delta\,\E[\Wk x_t]
=\mQ^\top R_\delta\mK$, using that the residuals are mean-zero. Hence
$\phi(\delta)=\rho^2(R_{-K}u)^\top R_\delta u=\rho^2 u^\top R_K R_\delta u
=\rho^2 u^\top R_{\delta+K}u$, where $R_{-K}^\top=R_K$ and
$R_KR_\delta=R_{\delta+K}$. By the block structure,
$u^\top R_s u=\sum_i (u_{2i-1},u_{2i})\Rot(s\theta_i)(u_{2i-1},u_{2i})^\top
=\sum_i (u_{2i-1}^2+u_{2i}^2)\cos(s\theta_i)=g(s)$. Since
$\sum_i p_i=\|u\|^2=1$ and $\cos\le1$, we get $g(s)\le1$ with equality iff
$\cos(s\theta_i)=1$ for every $i$ with $p_i>0$, that is, iff
$s\theta_i\in2\pi\mathbb{Z}$ for all such $i$. Setting $s=\delta+K$ completes
the proof.
\end{proof}
Lemma~\ref{lem:select} shows the carrier is \emph{sufficient} to select the
target. A converse holds: only the carrier geometry places the profile's
largest attainable value on the target, which turns
Assumption~\ref{ass:carrier} from a convenience into a structural fact.
\begin{proposition}[Necessity of the carrier]\label{prop:necessity}
For any $\Wq,\Wk$, the content-averaged logit equals the carrier profile,
\[
\E[\sqrt d\,\ell_t\mid\delta]=\mQ^\top R_\delta\mK=\textstyle\sum_i r_i\cos(\delta\theta_i-\psi_i),
\]
with $\mQ=\E[\Wq x_n]$, $\mK=\E[\Wk x_t]$, and $r_i\ge0,\psi_i$ the energy and
phase gap of the $i$th rotary pair of $(\mQ,\mK)$; the content-dependent
remainder is mean-zero. This profile attains its global maximum $\sum_i r_i$ at
the target offset $\delta=-K$ if and only if $\psi_i=-K\theta_i$ for every pair
with $r_i>0$, that is, iff $\mQ$ and $R_{-K}\mK$ are pairwise phase-aligned.
Any misalignment leaves the target strictly below $\sum_i r_i$ and forfeits
Lemma~\ref{lem:select}'s length-uniform guarantee that no competing offset
overtakes it; in this sense the geometry of Assumption~\ref{ass:carrier} is
forced, not a modeling convenience (proof in the supplement).
\end{proposition}
\begin{remark}\label{rem:residual}
The residual $\xi$ has zero mean, so it does not shift the argmax of the
\emph{expected} profile; under $\E[\xi\xi^\top]=\sigma^2(I-uu^\top)$ it
contributes only an offset-independent variance term. The excluded diagonal
offset $\delta=0$ compares the query token with itself, and its logit carries
an additional content-covariance term $\E[\xi_n'^\top\xi_n]$; this term is
what suppresses self-attention in the $K{=}1$ case discussed in
Section~\ref{sec:disc}. The empirical carrier fractions reported below show that the carrier term
dominates in practice.
\end{remark}
\paragraph{Empirical support.} The carrier energy
$\|\text{mean}\|^2/\E\|\text{proj}\|^2$ is $0.69/0.68$ at $K{=}1$, $0.82/0.86$ at
$K{=}2$, and $0.80/0.86$ at $K{=}3$ (keys/queries), against a
no-shared-direction baseline of $1/d=0.016$, and the carrier-only profile
peaks at exactly $-K$ for $K\in\{2,3\}$ (Fig.~\ref{fig:validation}). The
carriers are pre-rotated to the target by exactly the amount
Lemma~\ref{lem:select} predicts: $\cos(m_Q, R_{-K}m_K)=0.62, 0.83, 0.87$ for
$K=1,2,3$, while the unrotated alignment $\cos(m_Q,m_K)=0.71, 0.07, -0.40$
falls off as $K$ grows. At $K{=}1$ the target adjoins the query's own
position and the carriers are near-parallel ($\cos(m_Q,m_K)=0.71$, profile
peak at $0$), a special case we characterize separately
(Section~\ref{sec:disc}).

\subsection{Attention Dilution Explains the Accuracy Decay}
\begin{corollary}[Decay law]\label{cor:decay}
Let $\Delta_L=\{-(L{-}1),\dots,0\}$ and let
$\mu:=\min_{\delta\in\Delta_L,\,\delta\ne -K}\big(\phi(-K)-\phi(\delta)\big)>0$
be the profile margin. The target attention weight
$a_\star(L)=e^{\phi(-K)}/\sum_{\delta\in\Delta_L}e^{\phi(\delta)}$ satisfies
\[
a_\star(L)=\frac{1}{1+(L-1)\,\bar r_L}
\;\ge\;\frac{1}{1+(L-1)\,e^{-\mu}},
\]
where $\bar r_L$ is the average of $e^{\phi(\delta)-\phi(-K)}$ over the $L-1$
non-target offsets, and $a_\star(L)$ is strictly decreasing in $L$.
\end{corollary}
\begin{proof}
Dividing numerator and denominator by $e^{\phi(-K)}$ gives
$a_\star(L)=1/(1+\sum_{\delta\ne-K}e^{\phi(\delta)-\phi(-K)})$, which is the
stated exact form; each of the $L-1$ non-target terms is at most $e^{-\mu}$,
which gives the bound. Increasing $L$ by one adds a positive term to the
denominator, so $a_\star$ strictly decreases.
\end{proof}
Because the kernel tail flattens at far offsets, the average suppression
factor $\bar r_L$ is nearly length-independent in practice, so $a_\star(L)$
is well described by the two-parameter form $1/(\alpha+c(L-1))$ with fitted
constants. Least-squares fits of $1/a_\star$ against $L-1$ give
$R^2=0.914, 0.926, 0.943$ for $K=1,2,3$, with $\alpha$ near $1$ for $K\ge2$
(Table~\ref{tab:carrier}). The peak \emph{location} is fixed by
Lemma~\ref{lem:equiv} while the peak \emph{height} erodes by
Corollary~\ref{cor:decay}; this is the two-part structure of
Fig.~\ref{fig:phenomenon}. The analysis treats per-sequence logits through
their mean profile $\phi$, which the high carrier fraction justifies.

\subsection{Negative Control: Absolute Encodings Are Position-Pinned}
\begin{proposition}[Position pinning]\label{prop:ape}
Under APE the logit is proportional to $(x_n+p_n)^\top W(x_t+p_t)$ with
$W=\Wq^\top\Wk$ and $n=L-1$. Suppose training lengths never exceed
$L_{\max}$, so only $\{p_0,\dots,p_{L_{\max}-1}\}$ enter the training loss,
and suppose the position embeddings are optimized with weight decay
$\lambda\ge0$. Then: \emph{(i)} the loss is independent of
$\{p_j:j\ge L_{\max}\}$, so each such $p_j$ receives zero gradient at every
step; it remains at its random initialization if $\lambda=0$ and decays
geometrically toward $0$ if $\lambda>0$. \emph{(ii)} In either case, $p_j$
for $j\ge L_{\max}$ is statistically independent of the task and encodes
nothing about the target offset. \emph{(iii)} At any test length
$L'>L_{\max}+K$, both the query embedding $p_{L'-1}$ and the target-key
embedding $p_{L'-1-K}$ are untrained. The target's logit therefore contains
no learned positional term that selects it, and any position the head does
prefer is fixed by training and does not track the target as $L'$ varies, so
the two can coincide only by chance.
\end{proposition}
\begin{proof}
Training sequences occupy positions in $\{0,\dots,L_{\max}-1\}$, so $p_j$
with $j\ge L_{\max}$ never enters a forward pass and
$\partial\mathcal L/\partial p_j=0$ identically. Under (stochastic) gradient
descent with weight decay, the update for such $p_j$ is
$p_j\leftarrow(1-\eta\lambda)p_j$, which proves (i), and (ii) is immediate.
For (iii), note $L'-1>L'-1-K\ge L_{\max}$, so the target logit
$(x_{n'}+p_{L'-1})^\top W(x_{L'-1-K}+p_{L'-1-K})$, with $n'=L'-1$ the test
query position, involves positional vectors that are either zero or
independent of the task. Conditional on training, the logits at untrained
key positions are therefore exchangeable in distribution, over the i.i.d.\
tokens and, when $\lambda=0$, the task-independent initialization draws, so
no untrained position, the target included, is systematically favored.
\end{proof}
\begin{remark}
The failure depends on the out-of-range coordinates being \emph{untrained},
not on their magnitude. Empirically (Section~\ref{sec:exp}), out-of-range
$\|p_j\|$ stay at initialization for $\lambda=0$ and vanish for $\lambda>0$,
and the near-chance OOD accuracy is unchanged across
$\lambda\in\{0,10^{-4},10^{-2}\}$. The pin \emph{location} does depend on the
regularization: it concentrates near $L_{\max}-1-K$ under ridge and is lower
and more variable without it, but the failure itself persists.
Interpolating or rescaling the trained position table also fails to recover
accuracy. Consistent with this reading, randomizing position indices during
training so that indices beyond $L_{\max}$ are also trained is known to
restore length generalization \citep{ruoss2023randomized}; it removes
precisely the untrained-coordinate failure formalized here. For the same
reason, \citet{kazemnejad2023nope} adopt sinusoidal rather than learned
absolute encodings, since the learnable variant cannot produce embeddings for
unseen positions; the learned-APE collapse we study is that failure made
precise, whereas their fixed encodings are defined at every position and
degrade more gently instead of dropping to chance.
\end{remark}
Empirically, APE accuracy out of range is near chance, and all evaluated OOD
lengths satisfy $L'>L_{\max}+K$. We claim position pinning, not a specific
index: without weight decay the pin lies inside the trained range and varies
across seeds, while under ridge it tightens toward $L_{\max}-1-K$
(Section~\ref{sec:exp}).

\section{Experiments}
\label{sec:exp}
\begin{figure}[t]\centering
\includegraphics[width=0.8\linewidth]{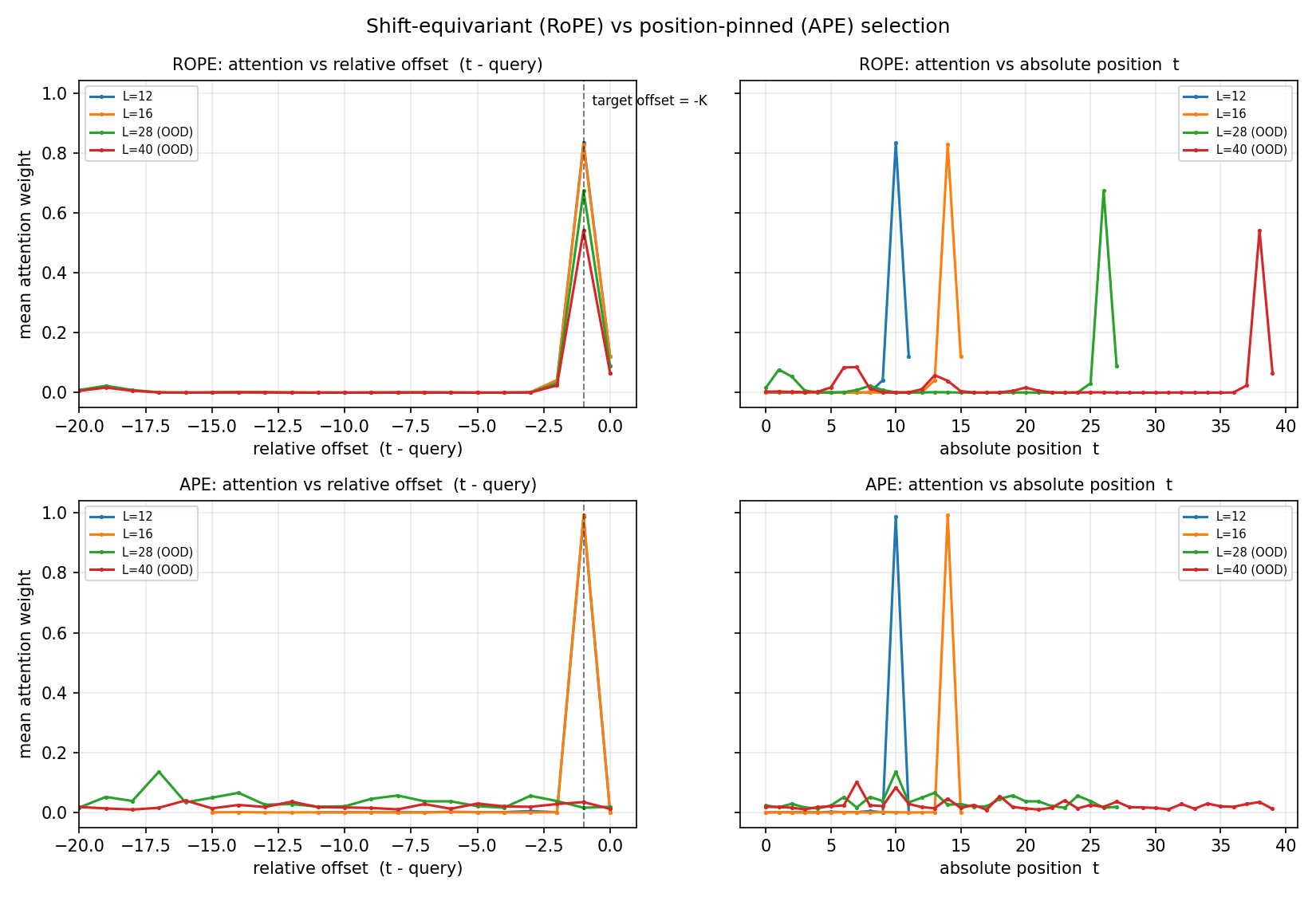}
\caption{Double dissociation. RoPE attention collapses onto a single profile
in \emph{relative} coordinates (left column); APE attention collapses in
\emph{absolute} coordinates (right column) and misses the target offset out
of distribution.}
\label{fig:dissociation}
\end{figure}
\begin{figure}[t]\centering
\includegraphics[width=0.95\linewidth]{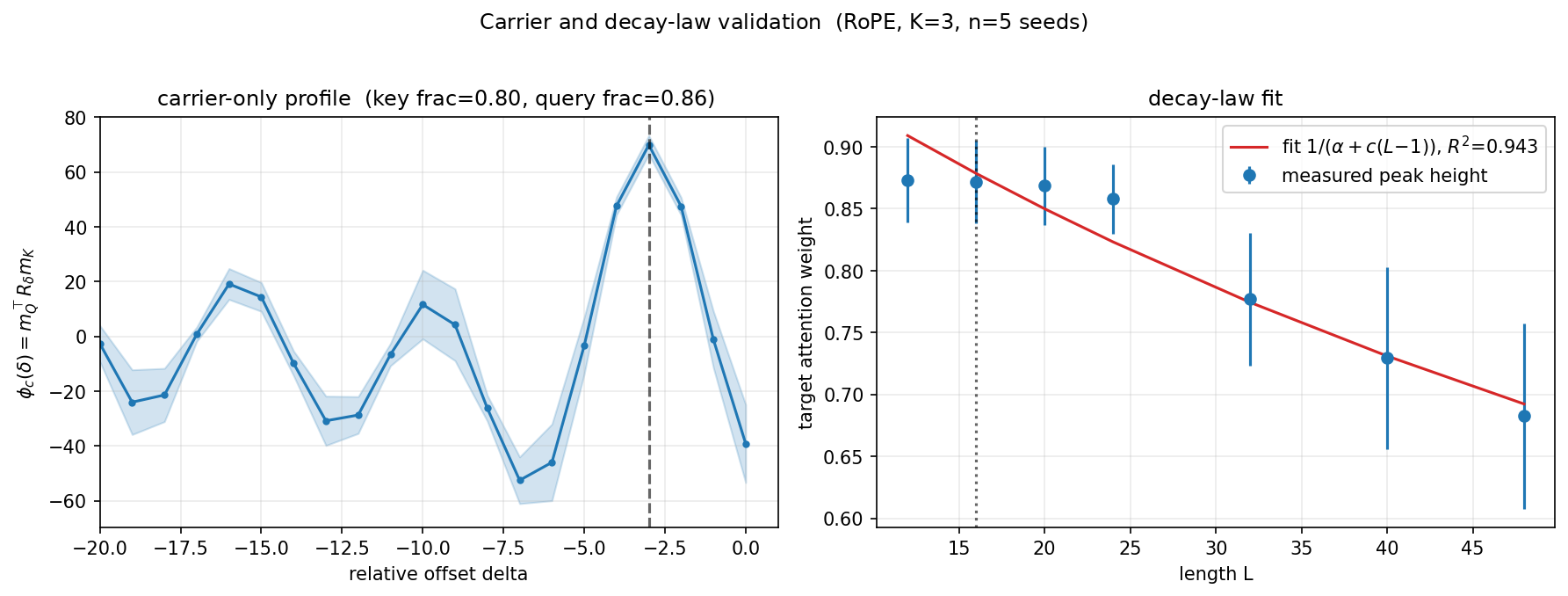}
\caption{Validation of Lemma~\ref{lem:select} and Corollary~\ref{cor:decay}
(RoPE, $K{=}3$, $5$ seeds). \textbf{Left:} the carrier-only profile peaks at
$-K$. \textbf{Right:} fit of $1/(\alpha+c(L{-}1))$ to the peak-height decay.}
\label{fig:validation}
\end{figure}

\paragraph{Architecture and task.} One layer, one head, $d=64$, vocabulary
size $m=20$, no MLP or layer normalization; the query is taken from the last
position; softmax attention; linear readout $U\in\R^{m\times d}$. RoPE uses
base $10^4$; APE uses a learned position embedding. We train on fixed-offset
retrieval with $K\in\{1,2,3\}$ and lengths uniform in
$[L_{\min},L_{\max}]=[8,16]$.

\paragraph{Training and evaluation.} Adam with learning rate $10^{-3}$,
batch size $256$, $4000$ steps, cross-entropy loss. We use $8$ seeds for the
phenomenon (Fig.~\ref{fig:phenomenon}, Table~\ref{tab:acc}) and $5$ seeds for
the validation study (Table~\ref{tab:carrier},
Figs.~\ref{fig:dissociation}--\ref{fig:validation}). Evaluation lengths are
$\{12,16,20,24,32,40,48\}$, with accuracy computed over $4096$ sequences and
attention probes averaged over $4096$ sequences; peak height is the maximum
over positions of the mean attention.

\paragraph{Metrics.} (i) accuracy; (ii) selected peak offset (argmax of mean
attention relative to the query); (iii) carrier energy
$\|\text{mean}\|^2/\E\|\text{proj}\|^2$ of projected keys and queries; (iv)
the carrier-only profile $\phi_c(\delta)=\mQ^\top R_\delta\mK$; (v) the decay
fit, linear least squares of $1/a_\star$ on $L-1$.

\paragraph{Compute.} All experiments run on CPU; no GPU is required. Seeds
are fixed to $0,\dots,n-1$, and our code, a single
script with one subcommand per experiment, will be released.

\begin{table}[t]\centering
\caption{Validation across offsets ($5$ seeds): key/query carrier fractions
versus the $1/d=0.016$ baseline, carrier-only profile peak, and decay-law
fit. At $K{=}1$ the carrier peak sits at the adjacent offset $0$
(Section~\ref{sec:disc}). Equivariance: selected offset $=-K$ in $100\%$ of
cells for all $K$.}
\label{tab:carrier}
\small
\setlength{\tabcolsep}{3.2pt}
\begin{tabular}{@{}ccccccc@{}}
\toprule
$K$ & key frac & query frac & carrier peak & $\alpha$ & $c$ & $R^2$\\
\midrule
1 & $0.69$ & $0.68$ & $0$ (adj.) & $0.82$ & $0.0220$ & $0.914$\\
2 & $0.82$ & $0.86$ & $-2$ & $1.06$ & $0.0067$ & $0.926$\\
3 & $0.80$ & $0.86$ & $-3$ & $1.00$ & $0.0096$ & $0.943$\\
\bottomrule
\end{tabular}
\end{table}

\begin{figure}[t]\centering
\includegraphics[width=\linewidth]{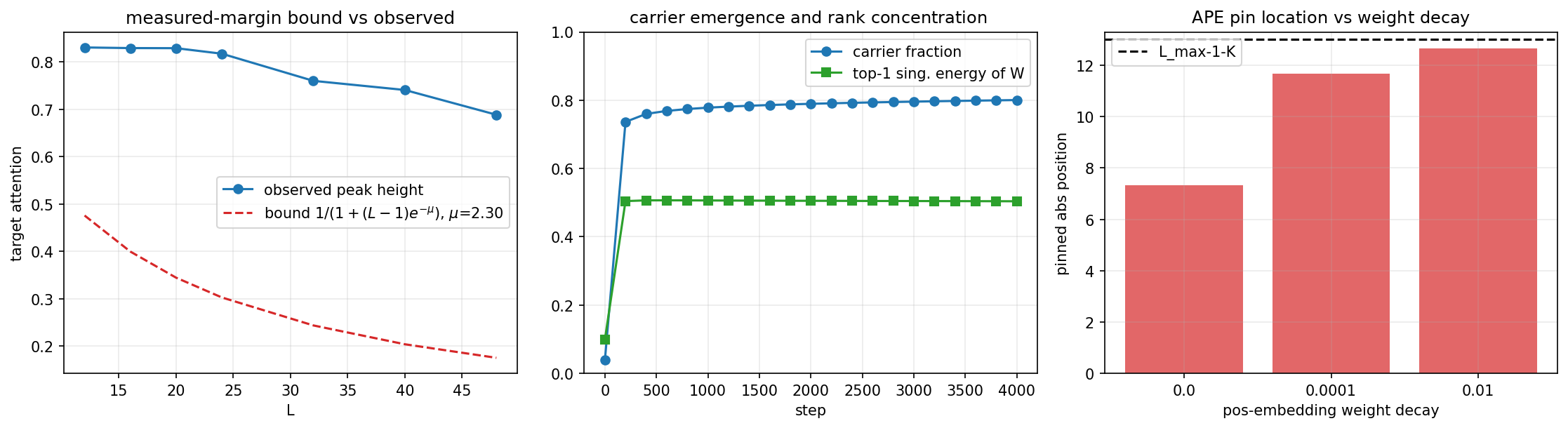}
\caption{Mechanism and robustness controls. \textbf{Left:} the
measured-margin lower bound versus observed target attention (valid but loose).
\textbf{Middle:} the carrier fraction and top-1 singular-value energy of $W$
emerge together early and plateau. \textbf{Right:} APE pin location versus
weight decay; OOD accuracy stays near chance throughout.}
\label{fig:reviewer}
\end{figure}

\paragraph{Robustness and mechanism checks.} Figure~\ref{fig:reviewer}
summarizes additional controls, with the full sweeps in the supplement. The
carrier and the rank concentration of $W$ co-emerge within a few hundred steps
and then plateau, matching the low-rank attention-as-SVM prediction, and the
measured margin $\mu$ gives a valid but loose bound $1/(1+(L-1)e^{-\mu})$ on
the target weight. The carrier fraction holds in $0.72$--$0.86$ across widths
$d\in\{32,64,128\}$, vocabulary sizes $\{10,20,40\}$, length ranges, and
softmax temperatures, while the dilution rate $c$ falls with width as wider
models sharpen the kernel. On the APE side the failure survives weight decay
$\lambda\in\{0,10^{-4},10^{-2}\}$, and inference-time fixes (table
interpolation, clamp-to-last, index rescaling) all leave accuracy at chance.
The mechanism is likewise robust to the rotary base (OOD accuracy
$1.00/1.00/0.97$ for base $\in\{10^3,10^4,10^5\}$, selecting $-K$
throughout), and the kernel $g(s)=\sum_j\cos(s\theta_j)$ has a unique integer
maximum at $s=0$ ($\max_{s\ge4} g(s)/g(0)\approx0.67$--$0.79$), so the
uniqueness condition of Lemma~\ref{lem:select} holds in practice.

\subsection{Other Relative Schemes and Content-Conditioned Selection}
\label{sec:relpe}
\begin{figure}[t]\centering
\includegraphics[width=0.8\linewidth]{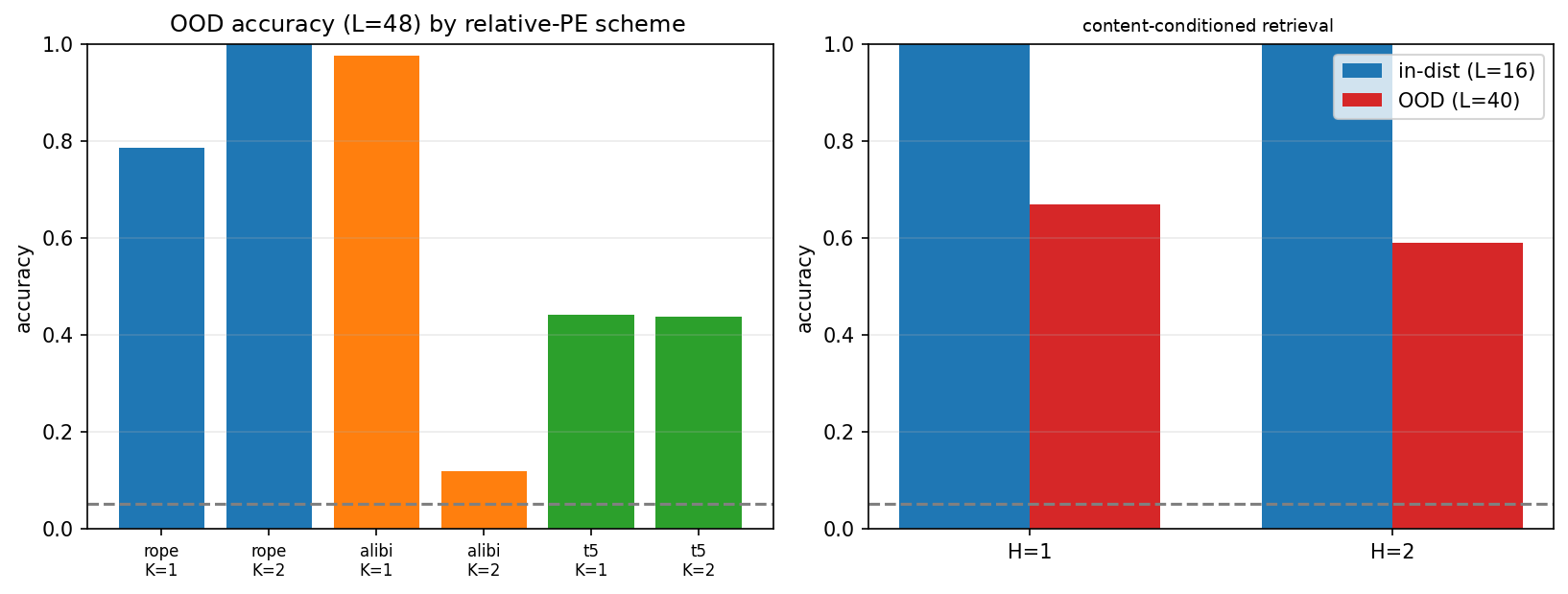}
\caption{\textbf{Left:} OOD accuracy ($L{=}48$) by relative-PE scheme and
offset. ALiBi solves $K{=}1$ but fails at $K{=}2$; T5 and RoPE select
both offsets, T5 with faster decay. \textbf{Right:} content-conditioned
retrieval; a single head ($H{=}1$) already solves it in distribution.}
\label{fig:reviewer2}
\end{figure}

\paragraph{The account generalizes across relative schemes.} We repeat the
fixed-offset experiment with ALiBi \citep{press2022alibi}, a monotonic distance
bias, and a T5-style learned per-offset bias (Fig.~\ref{fig:reviewer2}). All three are offset-equivariant
(Lemma~\ref{lem:equiv} holds for any logit depending only on $t-n$), so all
three transfer their rule, and they differ on the two axes the account
predicts. \emph{Capacity}: ALiBi solves $K{=}1$ (accuracy $\approx0.97$) but fails at
$K{=}2$ ($0.11$--$0.13$, chance $0.05$), selecting $-1$ because a fixed
locality bias cannot isolate a non-adjacent offset; the learnable T5 and
RoPE select any $K$.
\emph{Dilution}: T5 decays faster ($1.00$ to $0.44$) as far positions share its
clamped bucket bias, whereas RoPE stays at $1.00$. Offset selection plus
dilution is the shared mechanism, capacity and dilution rate the
scheme-dependent factors.

\paragraph{Larger offsets and content-conditioned selection.} The carrier
persists at larger offsets: the carrier-only profile peaks at exactly $-K$ for
$K\in\{2,3,4,5\}$. We also test a \emph{content-conditioned} task in which a
mode token at the query chooses which of two offsets ($K_0{=}1$ or $K_1{=}3$)
to retrieve. A single head solves it in distribution ($1.00$), and its
mode-conditioned carrier splits as predicted, peaking at $-3$ for mode $1$ and
at $0$ for mode $0$ (the $K{=}1$ adjacency of Section~\ref{sec:disc}). The
mechanism thus extends to data-dependent selection: the query content picks the
rotation alignment, and the residual term is signal, not noise around a fixed
offset.

\subsection{Bridging to a Realistic Architecture}
\label{sec:real}
\begin{figure}[t]\centering
\includegraphics[width=\linewidth]{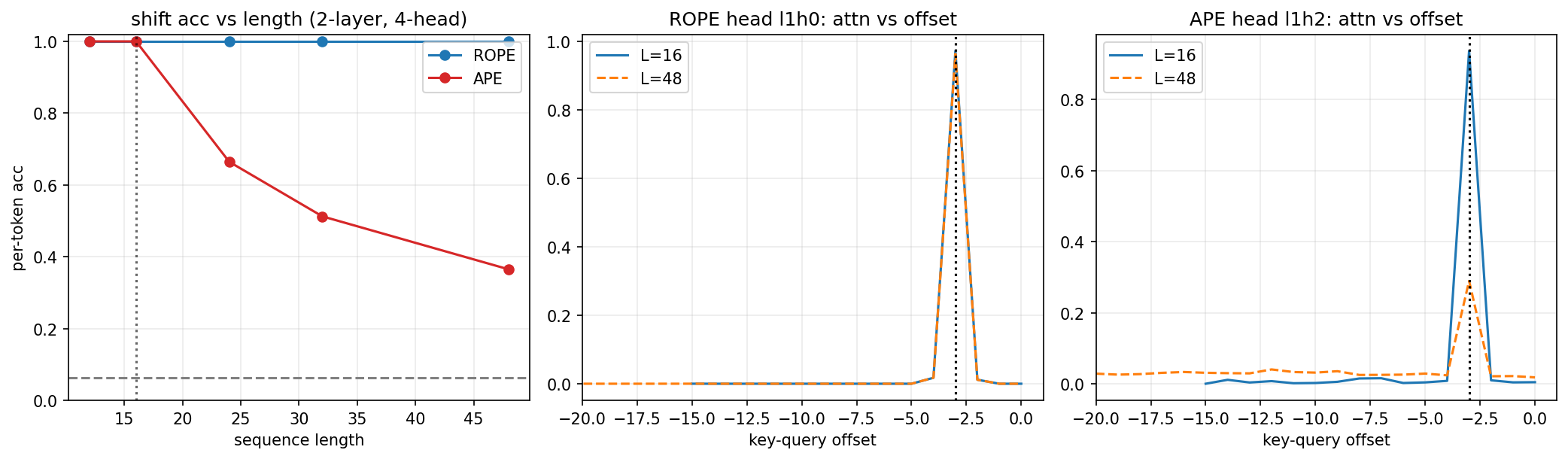}
\caption{The mechanism in a $2$-layer, $4$-head transformer with MLP and
LayerNorm on a full-sequence shift task
($\text{out}[i]=\text{in}[i{-}K]$, $K{=}3$), trained on lengths $\le16$.
\textbf{Left:} RoPE generalizes to $L{=}48$ ($1.00$) while learned APE degrades
out of range ($0.66/0.51/0.37$). \textbf{Middle:} the RoPE shift head's offset
profile coincides at $L{=}16$ and $L{=}48$ and peaks at $-K$, preserving
equivariance. \textbf{Right:} the APE head peaks at $-K$ in range but smears
toward chance at $L{=}48$.}
\label{fig:real}
\end{figure}
To test whether the mechanism is an artifact of the one-layer setting, we
train a decoder-only transformer with two layers, four heads, an MLP, and
LayerNorm on a full-sequence \emph{fixed-shift} task,
$\text{out}[i]=\text{in}[i{-}K]$ at every position, with content tokens drawn
from a small vocabulary so that the map is purely positional. All three findings transfer (Fig.~\ref{fig:real}).
The phenomenon reproduces: RoPE keeps $1.00$ accuracy out to $L{=}48$ after
training on lengths at most $16$, while APE falls to $0.37$. The APE decline
is position pinning read per token: it stays accurate at trained positions and
drops to chance beyond $L_{\max}$, so per-token accuracy tracks the in-range
fraction $\approx L_{\max}/L$ rather than collapsing outright, exactly as
Proposition~\ref{prop:ape} predicts. Equivariance
holds in a real head: the shift head's attention-versus-offset profiles at
$L{=}16$ and $L{=}48$ coincide and peak at $-K$ for RoPE, whereas the APE
head's profile smears toward chance out of range. The carrier survives depth,
heads, and the MLP: the RoPE shift head's key-input carrier fraction is
$0.976$ (baseline $1/d_h=0.021$), against $0.038$ for the failing APE head.
The selection mechanism is therefore not a one-layer artifact.

\section{Related Work}
\paragraph{Implicit bias of attention.} \citet{tarzanagh2023svm}
characterize the direction of the key--query product as a max-margin token
selector. We inherit that reduction but ask a question their analysis leaves
open: what \emph{positional} structure the selected solution has, and how it
behaves at unseen lengths. Their result says a token is selected; we show
which relative offset is selected, and that the offset transfers across
lengths.
\paragraph{Implicit bias for extrapolation in recurrent models.}
\citet{cohenkarlik2022temporal} show that although recurrent networks admit
many interpolating solutions that fail on longer sequences, gradient descent
from suitable initialization converges to solutions that extrapolate. Their
analysis is tied to the recurrent parameterization and does not transfer to
attention, where there is no recurrent operator and selection runs through a
softmax over positional logits. Our mechanism, offset equivariance of a
rotary kernel, is specific to attention with positional encoding.
\paragraph{Training dynamics of relative versus absolute encodings.}
Closest to our question, \citet{sabbaghi2024symmetry} prove for a one-layer
\emph{linear} transformer trained by gradient descent with vanishing weight
decay, on a translation-symmetric regression task, that a relative
parameterization length-generalizes while a learned absolute one drives its
out-of-range position vectors to zero and fails, the linear analogue of our
Proposition~\ref{prop:ape}. Their setting is
softmax-free and non-rotary, whereas the mechanism we isolate, offset
equivariance of the rotary kernel and the target-aligned carrier, is a
property of softmax attention. Their linear relative model \emph{succeeds},
whereas our linear-attention control (Section~\ref{sec:disc}) shows that
removing softmax destroys extrapolation once the rotary frequencies outnumber
the training offsets.
\paragraph{Length generalization: expressivity versus selection.}
The RASP-L conjecture \citep{zhou2023rasp} and the formal framework of
\citet{huang2025formal} establish whether a length-generalizing solution
\emph{exists} or is identifiable, the latter specifically for learnable
absolute positional encodings. We address the complementary
\emph{optimization} question of whether training selects such a solution, and
we exhibit a task where existence holds for both encodings yet only the
relative encoding's implicit bias selects the generalizing one.
\paragraph{Positional encodings and in-context selection.}
A large empirical literature designs positional encodings for better
extrapolation: rotary embeddings \citep{su2021roformer}, additive linear
biases \citep{press2022alibi}, randomized position indices
\citep{ruoss2023randomized}, task-structured position identifiers
\citep{cho2024position}, and the removal of explicit encodings so that
causal masking supplies position implicitly \citep{kazemnejad2023nope}. That
line of work asks \emph{which} scheme extrapolates. We ask \emph{why}, at the
level of the solution gradient descent selects, and our account
\emph{predicts} the relative-over-absolute ordering rather than measuring it.
Closest in spirit, \citet{barbero2024round} dissect how a trained language
model uses rotary frequencies, finding positional attention patterns built
from the highest frequencies; our kernel analysis characterizes how those
frequencies determine selection and dilution on a task where the link can be
verified exactly.
The existence-versus-selection distinction also surfaces beyond length
generalization: in in-context linear regression a model acts as a restricted
Bayesian estimator until pretraining task diversity crosses a threshold, then
implements the broader ridge solution \citep{raventos2023diversity}, again a
capable solution representable throughout, with training deciding whether it is
selected. Our ALiBi/T5/rotary comparison (Section~\ref{sec:relpe}) shows the
account is not RoPE-specific.

\section{Discussion and Limitations}
\label{sec:disc}
\paragraph{The $K{=}1$ special case.} For the immediate neighbor the carriers
are near-parallel (kernel centered at $0$) and suppressing the self-logit
suffices, via the content covariance of Remark~\ref{rem:residual}; the
explicit pre-rotation appears only for $K\ge2$, so we state the mechanism
claim for $K\ge2$ and treat $K{=}1$ separately.
\paragraph{Asymptotic selection.} The attention-SVM characterization is a
regularization-path or normalized-GD limit that finite-time training may not
reach; our experiments bridge that gap, and we characterize the \emph{selected
solution}, not the trajectory.
\paragraph{Scope.} Our analysis covers one layer, a single head, and
synthetic retrieval. Natural next steps are depth with NoPE, where causal
masking induces relative structure \citep{kazemnejad2023nope}; a true
two-layer \emph{induction} circuit \citep{olsson2022induction}, where a
content match rather than a query token sets the offset; and copying and
arithmetic benchmarks.
\paragraph{Multi-head behavior and capacity.} On the content-conditioned task
a single head implements both offsets by query-conditioned rotation; with two
heads the mode-$K_1$ selection becomes \emph{distributed}, solved at $1.00$
but no longer attributable to either head's mean carrier. Each head's $W$ is
dominated by but not equal to one direction (top singular-value energy
$0.39$--$0.57$), tempering a strict rank-one reading, and splitting width
across heads \emph{raises} dilution (peak attention at $L{=}48$ falls from
$0.69$ at $H{=}1$ to $0.30$ at $H{=}4$) as each head carries fewer
frequencies. The operative resource is the per-head frequency count; how the
carrier composes across heads remains open.
\paragraph{Reconciling with RoPE at scale.} Large rotary models degrade beyond
their training window, motivating remedies such as position interpolation
\citep{chen2023interpolation} and frequency-aware rescaling
\citep{peng2024yarn}. Our account separates two failure channels: equivariance
transfers the selection rule but does not preserve attention mass, and the
dilution law erodes the target weight at a rate set by kernel width. So the
account predicts not outright success at scale but degradation by dilution,
while absolute encodings fail structurally at the first unseen position,
matching the fragile, seed-sensitive length generalization seen in practice
\citep{zhou2024robust}.
\paragraph{Softmax normalization is essential.} Recasting selection as an
exactly solvable linear-attention least-squares problem removes the
asymptotic-SVM and shared-carrier assumptions, and does \emph{not} reproduce
the mechanism. Trained linear attention fits the in-distribution task by a
min-norm interpolation across rotary frequencies rather than a target-aligned
carrier: at $d{=}64$ the carrier inner product $\cos(\mQ, R_{-K}\mK)$ stays
near $0.1$ and the key carrier fraction drops to $0.07$ (baseline $0.016$),
against $0.83$ and $0.80$ under softmax. It does not extrapolate ($0.96$ in
distribution, $0.07$ at $L{=}48$, chance $0.05$, softmax $1.00$). The reason
is identifiability: the linear score profile
$f(\delta)=\sum_j A_j\cos(\delta\theta_j+\psi_j)$ is constrained only at the
$|\Delta|$ training offsets, so for $d>|\Delta|$ it is underdetermined
off-support, where the min-norm interpolant is generically $\Theta(1)$. Softmax removes this freedom: normalization couples all positions, so
the target's weight is set by the length-invariant margin $\mu$
(Corollary~\ref{cor:decay}) at every $L$. The carrier thus appears specific to
softmax; Assumption~\ref{ass:carrier} (confirmed in Section~\ref{sec:exp}) and
the asymptotic-SVM limit remain assumptions of Lemma~\ref{lem:select}.
\paragraph{Why the carrier forms.} Two questions separate: whether the
carrier \emph{must} appear, and whether gradient descent \emph{produces} it.
The first is settled by Proposition~\ref{prop:necessity}: since the zero-mean
residuals cannot move the expected argmax, content-independent selection
forces the carrier geometry, so it is structural rather than incidental. The
second, whether Adam on the cross-entropy loss lands on that geometry, we
leave open, though the early co-emergence of the carrier and the rank
concentration of $W$ (Section~\ref{sec:exp}) indicates that it does.
Extending analyses of global max-margin convergence for the key--query
product \citep{vasudeva2024implicit} to positional logits and the
cross-entropy readout is a natural route.

\section{Conclusion}
On a task isolating positional selection, length generalization in
attention is governed by implicit bias: rotary encodings make selection
offset-equivariant, so the rule transfers exactly, while absolute encodings
pin to the training range. The selected rule is a target-aligned carrier whose
accuracy erodes by dilution, long reported and now explained at the level of
the selected solution.

\bibliographystyle{plainnat}
\bibliography{references}

\clearpage
\appendix
\section*{Appendix}

This appendix collects the complete proofs (Appendix~\ref{app:proofs}),
the linear-attention control that isolates the role of softmax
(Appendix~\ref{app:linear}), the full robustness and mechanism sweeps
summarized in the main paper (Appendix~\ref{app:sweeps}), the multi-head and
content-conditioned measurements (Appendix~\ref{app:multihead}), and the
experimental and reproducibility details (Appendix~\ref{app:details}).
Numbered results (Lemma~1, Lemma~2, Proposition~1, etc.) refer to the main
paper. Unless noted, tables report values from the released script on the
reference setup ($d{=}64$, vocabulary $20$, $[L_{\min},L_{\max}]{=}[8,16]$,
Adam at $10^{-3}$, batch $256$, $4000$ steps), which match the values
reported in the main text.

\section{Complete Proofs}
\label{app:proofs}

Throughout, $R_t$ is the block-diagonal rotary matrix whose $i$th $2\times2$
block is $\Rot(t\theta_i)$, with inverse frequencies
$\theta_i=b^{-2(i-1)/d}$ for base $b$. We use two elementary facts about
planar rotations: $\Rot(\alpha)^\top\Rot(\beta)=\Rot(\beta-\alpha)$, and for
any $w\in\R^2$ and angle $\gamma$, $w^\top\Rot(\gamma)w=\|w\|^2\cos\gamma$
(the antisymmetric part of $\Rot(\gamma)$ contributes nothing to the
quadratic form). Block-diagonality lifts both facts to $R_t$.

\subsection{Lemma 1 (Equivariance)}
\begin{lemmaS}
For any $\Wq,\Wk$, the RoPE logit between a query at position $n$ and a key at
position $t$ depends on $(n,t)$ only through $\delta=t-n$:
$\sqrt d\,\ell_t=x_n^\top\Wq^\top R_\delta\Wk x_t$.
\end{lemmaS}
\begin{proof}
The head forms $q=R_n\Wq x_n$ and $k_t=R_t\Wk x_t$, so
$\sqrt d\,\ell_t=q^\top k_t=(R_n\Wq x_n)^\top(R_t\Wk x_t)
=x_n^\top\Wq^\top R_n^\top R_t\Wk x_t$. Blockwise,
$R_n^\top R_t$ has $i$th block $\Rot(n\theta_i)^\top\Rot(t\theta_i)
=\Rot((t-n)\theta_i)$, i.e.\ $R_n^\top R_t=R_{t-n}=R_\delta$. Substituting
gives the claim; no term depends on $n$ or $t$ except through $\delta$.
\end{proof}

\subsection{Lemma 2 (Target-aligned selection)}
\begin{assumptionS}[Shared carrier]
There exist a unit vector $u$ and a scalar $\rho>0$ with
$\Wk x_t=\rho u+\xi_t$ and $\Wq x_n=\rho R_{-K}u+\xi'_n$, where
$\E[\xi_t]=\E[\xi'_n]=0$ and $\xi_t\perp u$.
\end{assumptionS}
\begin{lemmaS}
Under the shared-carrier assumption, the content-averaged score profile over
keys $t\neq n$ (offsets $\delta\le-1$) is
$\phi(\delta)=\E[\sqrt d\,\ell_t\mid\delta]=\mQ^\top R_\delta\mK
=\rho^2 g(\delta+K)$, where $\mK=\rho u$, $\mQ=\rho R_{-K}u$,
$g(s)=\sum_{i=1}^{d/2}p_i\cos(s\theta_i)$, and $p_i=u_{2i-1}^2+u_{2i}^2\ge0$
with $\sum_i p_i=1$. Then $g(s)\le g(0)=1$, so $\phi$ is maximized at
$\delta=-K$; if no nonzero integer $s$ in the evaluated range satisfies
$s\theta_i\in2\pi\mathbb Z$ for all $i$ with $p_i>0$, the maximizer is unique.
\end{lemmaS}
\begin{proof}
For $t\neq n$ the tokens $x_n,x_t$ are independent, and by Lemma~1
$\sqrt d\,\ell_t=(\Wq x_n)^\top R_\delta(\Wk x_t)$. Taking expectations and
using $\E[a^\top M b]=\E[a]^\top M\,\E[b]$ for independent $a,b$,
$\phi(\delta)=\mQ^\top R_\delta\mK$ with $\mQ=\E[\Wq x_n]=\rho R_{-K}u$,
$\mK=\E[\Wk x_t]=\rho u$ (the residuals are mean-zero). Hence
$\phi(\delta)=\rho^2(R_{-K}u)^\top R_\delta u=\rho^2 u^\top R_K R_\delta u
=\rho^2 u^\top R_{\delta+K}u$, using $R_{-K}^\top=R_K$ and
$R_KR_\delta=R_{\delta+K}$. Blockwise,
$u^\top R_s u=\sum_i(u_{2i-1}^2+u_{2i}^2)\cos(s\theta_i)=g(s)$. Since
$\sum_i p_i=\|u\|^2=1$ and $\cos\le1$, $g(s)\le1$ with equality iff
$s\theta_i\in2\pi\mathbb Z$ for every $i$ with $p_i>0$. Setting $s=\delta+K$
gives the maximizer $\delta=-K$ and the stated uniqueness condition.
\end{proof}

\subsection{Proposition 1 (Necessity of the carrier)}
This is the converse of Lemma~2: the target-aligned carrier geometry is the
only one whose content-averaged profile attains its largest possible value at
the target, and the only one that keeps Lemma~2's length-uniform guarantee
that no competing offset overtakes it. In that sense Assumption~1 is forced,
not a modeling convenience.
\begin{propositionS}
For any $\Wq,\Wk$, set $\mQ=\E[\Wq x_n]$ and $\mK=\E[\Wk x_t]$. The
content-averaged logit equals the carrier profile,
\[
\E[\sqrt d\,\ell_t\mid\delta]=\mQ^\top R_\delta\mK
=\sum_{i=1}^{d/2}r_i\cos(\delta\theta_i-\psi_i),
\]
where $r_i=\|a_i\|\,\|b_i\|\ge0$ and $\psi_i=\angle a_i-\angle b_i$ are the
energy and phase gap of the $i$th rotary pair $(a_i,b_i)$ of $(\mQ,\mK)$, and
the content-dependent remainder is mean-zero. The profile attains its global
maximum $\sum_i r_i$ at the target offset $\delta=-K$ if and only if
$\psi_i=-K\theta_i$ for every $i$ with $r_i>0$, i.e.\ iff $\mQ$ and
$R_{-K}\mK$ are pairwise phase-aligned. Under the same incommensurability
condition as Lemma~2, $-K$ is then the unique maximizer.
\end{propositionS}
\begin{proof}
\emph{(Content-averaged logit.)} Write $\Wk x_t=\mK+\xi_t$ and
$\Wq x_n=\mQ+\xi'_n$ with $\E[\xi_t]=\E[\xi'_n]=0$. By Lemma~1,
$\sqrt d\,\ell_t=(\mQ+\xi'_n)^\top R_\delta(\mK+\xi_t)$. For $\delta\le-1$ the
tokens are independent, so taking expectations kills the three terms carrying
a lone mean-zero factor and leaves $\E[\sqrt d\,\ell_t\mid\delta]=\mQ^\top
R_\delta\mK$. The remainder $\sqrt d\,\ell_t-\mQ^\top R_\delta\mK$ has zero
mean, hence cannot shift the argmax of the expected profile.

\emph{(Polar form.)} Let $a_i,b_i\in\R^2$ be the $i$th rotary pairs of
$\mQ,\mK$. Since $R_\delta$ acts on pair $i$ by $\Rot(\delta\theta_i)$ and
$a_i^\top\Rot(\gamma)b_i=\|a_i\|\|b_i\|\cos(\angle a_i-\angle b_i-\gamma)$,
\[
\mQ^\top R_\delta\mK=\sum_i a_i^\top\Rot(\delta\theta_i)b_i
=\sum_i r_i\cos(\delta\theta_i-\psi_i),
\]
with $r_i=\|a_i\|\|b_i\|$ and $\psi_i=\angle a_i-\angle b_i$ (using that
$\cos$ is even).

\emph{(Maximizer.)} Each cosine is at most $1$, so $\phi(\delta):=\mQ^\top
R_\delta\mK\le\sum_i r_i$ for all $\delta$, with equality at a given $\delta$
iff $\delta\theta_i-\psi_i\in2\pi\mathbb Z$ for every active pair
($r_i>0$). At $\delta=-K$ this reads $\psi_i=-K\theta_i$, i.e.\
$\angle a_i=\angle b_i-K\theta_i=\angle(\Rot(-K\theta_i)b_i)$, which is
exactly $\mQ$ and $R_{-K}\mK$ pairwise phase-aligned: the geometry of
Assumption~1 (there $a_i=\rho\,\Rot(-K\theta_i)u_i$ and $b_i=\rho u_i$, giving
$\psi_i=-K\theta_i$ and $r_i=\rho^2 p_i$, so $\phi(\delta)=\rho^2 g(\delta+K)$
as in Lemma~2). Under the incommensurability condition, every $\delta\neq-K$
in range leaves some active cosine below $1$, so $\phi(\delta)<\sum_i r_i$ and
$-K$ is unique. Because $\phi$ is content-independent, this maximizer is the
same at every length (Corollary~1), which is what lets content-independent
selection transfer.
\end{proof}
The residual can be quantified. Under
$\E[\xi\xi^\top]=\sigma^2(I-uu^\top)$ the fluctuation term contributes an
offset-independent variance $\sigma^2\rho^2$ to each logit, so
content-independent selection additionally requires the profile margin
$\mu=\min_{\delta\neq-K}(\phi(-K)-\phi(\delta))$ to dominate the fluctuation
scale; the measured carrier fractions (Table~\ref{tab:sweep}) and margin
(Table~\ref{tab:margin}) confirm this holds in practice.

\subsection{Corollary 2 (Decay law)}
\setcounter{corollaryS}{1}
\begin{corollaryS}
Let $\Delta_L=\{-(L-1),\dots,0\}$ and
$\mu=\min_{\delta\in\Delta_L,\delta\neq-K}(\phi(-K)-\phi(\delta))>0$. The
target attention weight
$a_\star(L)=e^{\phi(-K)}/\sum_{\delta\in\Delta_L}e^{\phi(\delta)}$ satisfies
$a_\star(L)=1/(1+(L-1)\bar r_L)\ge1/(1+(L-1)e^{-\mu})$, where $\bar r_L$
averages $e^{\phi(\delta)-\phi(-K)}$ over the $L-1$ non-target offsets, and
$a_\star$ is strictly decreasing in $L$.
\end{corollaryS}
\begin{proof}
Divide numerator and denominator by $e^{\phi(-K)}$:
$a_\star(L)=1/(1+\sum_{\delta\neq-K}e^{\phi(\delta)-\phi(-K)})$, the exact
form with $\bar r_L=\frac{1}{L-1}\sum_{\delta\neq-K}e^{\phi(\delta)-\phi(-K)}$.
Each of the $L-1$ terms is at most $e^{-\mu}$, giving the bound. Increasing
$L$ by one appends a positive term to the denominator, so $a_\star$ strictly
decreases. When the kernel tail flattens, $\bar r_L$ is nearly constant in
$L$, and $a_\star(L)\approx1/(\alpha+c(L-1))$ with $\alpha\approx1$; the fits
across the sweep of Table~\ref{tab:sweep} give $R^2\in[0.72,0.94]$, with
$R^2=0.914,0.926,0.943$ at the reference setup for $K=1,2,3$ (main paper).
\end{proof}

\subsection{Proposition 2 (Position pinning)}
\begin{propositionS}
Under APE the logit is proportional to $(x_n+p_n)^\top W(x_t+p_t)$ with
$W=\Wq^\top\Wk$ and $n=L-1$. If training lengths never exceed $L_{\max}$ and
the position table is optimized with weight decay $\lambda\ge0$, then:
\emph{(i)} the loss is independent of $\{p_j:j\ge L_{\max}\}$, so each such
$p_j$ has zero gradient and stays at initialization ($\lambda=0$) or decays
geometrically to $0$ ($\lambda>0$); \emph{(ii)} such $p_j$ is independent of
the task; \emph{(iii)} at any test length $L'>L_{\max}+K$, both $p_{L'-1}$ and
$p_{L'-1-K}$ are untrained, so no learned positional term selects the target,
which can be selected only by chance.
\end{propositionS}
\begin{proof}
Training sequences occupy positions $\{0,\dots,L_{\max}-1\}$, so for
$j\ge L_{\max}$ the vector $p_j$ never enters a forward pass and
$\partial\mathcal L/\partial p_j\equiv0$. The (stochastic) gradient step with
weight decay is $p_j\leftarrow(1-\eta\lambda)p_j$, which gives (i); (ii) is
immediate since $p_j$ is then a function only of its initialization. For
(iii), $L'>L_{\max}+K$ implies $L'-1>L'-1-K\ge L_{\max}$, so the target logit
$(x_{n'}+p_{L'-1})^\top W(x_{L'-1-K}+p_{L'-1-K})$ involves only untrained (zero
or task-independent) position vectors. Conditioned on training, the logits at
untrained key positions are exchangeable over the i.i.d.\ tokens and, when
$\lambda=0$, the i.i.d.\ initialization draws; no untrained position, the
target included, is systematically favored, so out-of-range accuracy is at
chance.
\end{proof}
On the full-sequence shift task this is visible per token: APE predicts
correctly at positions whose query and key both lie below $L_{\max}$ and at
chance elsewhere, so per-token accuracy tracks the in-range fraction
$\approx L_{\max}/L$ (Section~\ref{app:details}), which is why the two-layer
model falls to $0.37$ rather than to chance.

\section{The Linear-Attention Control}
\label{app:linear}
To test whether softmax normalization is essential, we replace
$\mathrm{softmax}(\ell)_t$ with the raw score $\ell_t/L$ (un-normalized linear
attention) and train the otherwise identical one-head RoPE model on
fixed-offset retrieval ($K{=}2$). Table~\ref{tab:linear} reports accuracy,
the key carrier fraction $\|\mathrm{mean}\|^2/\E\|\mathrm{proj}\|^2$, and the
carrier alignment $\cos(\mQ,R_{-K}\mK)$, for softmax and linear attention at
three widths.

\begin{table}[t]\centering
\caption{Softmax vs.\ linear attention on fixed-offset retrieval ($K{=}2$).
ID $=$ in-distribution accuracy ($L{=}16$); OOD at $L{=}32,48$. Softmax builds
the target-aligned carrier and extrapolates; linear attention solves ID by
min-norm interpolation, has no carrier ($\cos\approx0$), and collapses to
near chance (chance $0.05$) out of distribution. With few frequencies ($d{=}16$) linear
attention cannot even fit ID. ``frac'' is the key carrier fraction and
``align.''\ is $\cos(\mQ,R_{-K}\mK)$.}
\label{tab:linear}
\small
\setlength{\tabcolsep}{3pt}
\begin{tabular}{@{}llccccc@{}}
\toprule
attn & $d$ & ID & OOD$_{32}$ & OOD$_{48}$ & frac & align.\\
\midrule
softmax & 16 & $1.00$ & $0.64$ & $0.41$ & $0.79$ & $0.83$\\
softmax & 64 & $1.00$ & $1.00$ & $1.00$ & $0.80$ & $0.83$\\
softmax & 128& $1.00$ & $1.00$ & $1.00$ & $0.81$ & $0.92$\\
linear  & 16 & $0.56$ & $0.30$ & $0.24$ & $0.04$ & $-0.76$\\
linear  & 64 & $0.96$ & $0.09$ & $0.07$ & $0.07$ & $0.11$\\
linear  & 128& $0.96$ & $0.11$ & $0.08$ & $0.01$ & $-0.04$\\
\bottomrule
\end{tabular}
\end{table}

The failure has a clean identifiability reading. With RoPE the score is a
function of offset alone, $f(\delta)=\sum_j A_j\cos(\delta\theta_j+\psi_j)$,
and the training loss constrains $f$ only at the $|\Delta|$ in-distribution
offsets. For $d>|\Delta|$ the profile is underdetermined off the training
support, and the min-norm interpolant selected by training is generically
$\Theta(1)$ there; hence linear attention interpolates the training offsets
but diverges at unseen ones ($d{=}64,128$ in Table~\ref{tab:linear}), while
with too few frequencies it cannot localize the target at all ($d{=}16$). Softmax
removes this freedom: normalization couples all positions, so the target's
weight is governed by the length-invariant margin $\mu$ (Corollary~2) at every
$L$, and the carrier emerges. The carrier alignment separates the two regimes
sharply: $0.83$--$0.92$ under softmax versus $|\!\cos|\le0.11$ under linear
attention at $d\ge64$.

\section{Robustness and Mechanism Sweeps}
\label{app:sweeps}

\paragraph{Carrier fraction and decay law.} Table~\ref{tab:sweep} varies width,
vocabulary, training-length range, and softmax temperature about the reference
setup. The key carrier fraction stays in $0.72$--$0.86$ throughout, and the
inverse-linear decay fit $1/a_\star(L)=\alpha+c(L-1)$ holds with intercept
$\alpha\approx1$. The dilution rate $c$ falls with width, because wider models
carry more rotary frequencies and a sharper kernel; the one low $R^2$
($0.72$) is the wide training range $[8,32]$, where in-distribution lengths
already sample much of the evaluation range and the fit has little dynamic
range.

\begin{table}[t]\centering
\caption{Carrier fraction and decay-law fit across configurations (RoPE,
$K{=}2$, one seed per row). Baseline carrier fraction $1/d$.}
\label{tab:sweep}
\small
\setlength{\tabcolsep}{4pt}
\begin{tabular}{@{}llcccc@{}}
\toprule
knob & value & carr.\ frac & $\alpha$ & $c$ & $R^2$\\
\midrule
width $d$ & $32$ & $0.84$ & $0.80$ & $0.0171$ & $0.82$\\
          & $64$ & $0.80$ & $1.09$ & $0.0070$ & $0.93$\\
          & $128$& $0.81$ & $1.27$ & $0.0004$ & $0.91$\\
vocab $m$ & $10$ & $0.81$ & $1.14$ & $0.0060$ & $0.93$\\
          & $20$ & $0.80$ & $1.09$ & $0.0070$ & $0.93$\\
          & $40$ & $0.82$ & $1.06$ & $0.0072$ & $0.94$\\
range     & $[8,16]$ & $0.80$ & $1.09$ & $0.0070$ & $0.93$\\
          & $[4,12]$ & $0.78$ & $1.01$ & $0.0135$ & $0.94$\\
          & $[8,32]$ & $0.82$ & $1.15$ & $0.0021$ & $0.72$\\
temp $\tau$ & $0.5$ & $0.86$ & $1.15$ & $0.0085$ & $0.93$\\
          & $1.0$ & $0.80$ & $1.09$ & $0.0070$ & $0.93$\\
          & $2.0$ & $0.73$ & $1.01$ & $0.0075$ & $0.94$\\
\bottomrule
\end{tabular}
\end{table}

\paragraph{Rotary base and kernel incommensurability.}
Table~\ref{tab:base} sweeps the rotary base. The kernel
$g(s)=\sum_j\cos(s\theta_j)$ has a unique integer maximum at $s{=}0$; the
largest competing value at $s\ge4$ stays at $0.67$--$0.79$ of $g(0)$, so the
uniqueness condition of Lemma~2 holds empirically. OOD accuracy is
$1.00/1.00/0.97$ for base $\in\{10^3,10^4,10^5\}$ with selection at $-K$
throughout; larger bases add low frequencies that broaden the kernel,
explaining the mild degradation at $10^5$.

\begin{table}[t]\centering
\caption{Rotary base sweep ($K{=}2$). $g(1)/g(0)$ measures local kernel
breadth; $\max_{s\ge4}g(s)/g(0)$ measures far re-alignment. OOD accuracy at
$L{=}48$; selected offset at $L{=}32$.}
\label{tab:base}
\small
\setlength{\tabcolsep}{5pt}
\begin{tabular}{@{}lcccc@{}}
\toprule
base & $g(1)/g(0)$ & $\max_{s\ge4}g(s)/g(0)$ & OOD$_{48}$ & sel.\\
\midrule
$10^3$ & $0.96$ & $0.67$ & $1.00$ & $-K$\\
$10^4$ & $0.97$ & $0.75$ & $1.00$ & $-K$\\
$10^5$ & $0.97$ & $0.79$ & $0.97$ & $-K$\\
\bottomrule
\end{tabular}
\end{table}

\paragraph{APE regularization, pinning, and inference-time fixes.}
Table~\ref{tab:ape} confirms Proposition~2. Out-of-range embeddings stay at
initialization for $\lambda=0$ ($\|p_{\mathrm{out}}\|\approx\|p_{\mathrm{in}}\|$)
and vanish for $\lambda>0$; OOD accuracy stays near chance throughout; the pin
tightens toward $L_{\max}-1-K=13$ under ridge and is lower and more variable
without it. Table-interpolation, clamp-to-last, and index-rescaling all leave
OOD accuracy at chance ($0.05$--$0.06$), because rescaling distorts the very
offset structure the task requires.

\begin{table}[t]\centering
\caption{APE weight decay, pinning, and out-of-range norms ($K{=}2$, $3$
seeds; $L_{\max}{=}16$, so $L_{\max}{-}1{-}K=13$). Chance $0.05$.}
\label{tab:ape}
\small
\setlength{\tabcolsep}{4pt}
\begin{tabular}{@{}lccccc@{}}
\toprule
$\lambda$ & OOD acc & pin & $\|p_{\mathrm{in}}\|$ & $\|p_{\mathrm{out}}\|$ & interp.\ acc\\
\midrule
$0$      & $0.08$ & $7.3\pm1.7$  & $8.18$ & $8.08$ & $0.05$\\
$10^{-4}$& $0.10$ & $11.7\pm1.2$ & $3.83$ & $0.20$ & $0.05$\\
$10^{-2}$& $0.09$ & $12.7\pm0.5$ & $0.66$ & $0.20$ & $0.06$\\
\bottomrule
\end{tabular}
\end{table}

\paragraph{Measured-margin bound.} Table~\ref{tab:margin} evaluates the lower
bound of Corollary~2 with the empirically measured margin $\mu=2.30$. The
bound is valid at every length (observed $\ge$ bound) but loose, since it
retains only the single nearest competitor.

\begin{table}[t]\centering
\caption{Measured-margin lower bound $1/(1+(L-1)e^{-\mu})$, $\mu=2.30$, vs.\
observed peak attention (RoPE, $K{=}2$).}
\label{tab:margin}
\small
\setlength{\tabcolsep}{4.5pt}
\begin{tabular}{@{}lccccccc@{}}
\toprule
$L$ & 12 & 16 & 20 & 24 & 32 & 40 & 48\\
\midrule
bound    & $.48$ & $.40$ & $.35$ & $.30$ & $.24$ & $.20$ & $.18$\\
observed & $.83$ & $.83$ & $.83$ & $.82$ & $.76$ & $.74$ & $.69$\\
\bottomrule
\end{tabular}
\end{table}

\paragraph{Training dynamics.} Logged over training, the key carrier fraction
rises from $0.04$ at initialization to $0.74$ by step $200$ and plateaus near
$0.80$; the top-1 singular-value energy of $W=\Wq^\top\Wk$ rises from $0.10$
to $0.50$ over the same window and holds. The carrier and the rank
concentration therefore co-emerge within a few hundred steps, consistent with
the low-rank attention-as-SVM prediction.

\section{Multi-Head and Content-Conditioned Selection}
\label{app:multihead}

\paragraph{Multi-head capacity.} Splitting a fixed width across more heads
gives each head fewer rotary frequencies and a broader kernel, so dilution
\emph{increases} with head count. Table~\ref{tab:heads} reports per-head
top-1 singular-value energy of $W$ and peak attention by length for $H{=}1,4$;
peak attention at $L{=}48$ falls from $0.69$ to $0.30$. Each head's $W$ is
dominated by, but not equal to, a single direction (top-1 energy
$0.39$--$0.57$), tempering a strict rank-one reading.

\begin{table}[t]\centering
\caption{Per-head rank and dilution ($K{=}2$). Peak attention is averaged over
batch and heads.}
\label{tab:heads}
\small
\setlength{\tabcolsep}{4pt}
\begin{tabular}{@{}lcccccc@{}}
\toprule
& top-1 energy & \multicolumn{5}{c}{peak attention at $L=$}\\
\cmidrule(l){3-7}
$H$ & (per head) & 12 & 16 & 24 & 32 & 48\\
\midrule
$1$ & $0.50$ & $0.83$ & $0.83$ & $0.82$ & $0.76$ & $0.69$\\
$4$ & $0.39$--$0.57$ & $0.67$ & $0.63$ & $0.44$ & $0.42$ & $0.30$\\
\bottomrule
\end{tabular}
\end{table}

\paragraph{Content-conditioned retrieval.} A mode token at the query position
selects offset $K_0{=}1$ or $K_1{=}3$. A single head ($H{=}1$) solves the task
in distribution ($1.00$ accuracy); its mode-conditioned carrier splits as
predicted, peaking at $0$ for mode $0$ (the $K{=}1$ adjacency) and at $-3$ for
mode $1$. With $H{=}2$ the model still reaches $1.00$ in distribution, but the
mode-$K_1$ selection becomes distributed and is no longer attributable to
either head's mean carrier.

\paragraph{Larger offsets and other relative schemes.} The carrier-only
profile peaks at exactly $-K$ for $K\in\{2,3,4,5\}$. On the fixed-offset task,
ALiBi holds accuracy $\approx0.97$ across lengths at $K{=}1$ but fails at
$K{=}2$ (accuracy $0.11$--$0.13$ against chance $0.05$, selecting $-1$); a
T5-style learned bias selects any $K$ but decays faster ($1.00\to0.44$); RoPE
selects any $K$ and decays slowly ($K{=}2$ stays at $1.00$).

\section{Experimental Details and Reproducibility}
\label{app:details}

\paragraph{Model and task.} One layer, one head, $d{=}64$, vocabulary
$m{=}20$, no MLP or layer normalization; the query is the last position;
softmax attention; linear readout $U\in\R^{m\times d}$. RoPE uses base $10^4$;
APE uses a learned position table of size $\max L'{+}1$. The fixed-offset task
draws length-$L$ sequences of i.i.d.\ uniform tokens and labels each with the
token $K$ positions before the query. The two-layer model
(Section~5.2, main) is a decoder-only transformer with two layers, four heads,
$d{=}192$, an MLP of width $4d$ with GELU, LayerNorm, and a causal mask,
trained on the full-sequence shift $\mathrm{out}[i]=\mathrm{in}[i{-}K]$ with
$K{=}3$.

\paragraph{Training.} Adam at learning rate $10^{-3}$, batch $256$, $4000$
steps, cross-entropy loss (one-layer); $3\times10^{-4}$, batch $256$, $8000$
steps (two-layer). Training lengths are uniform in $[8,16]$. The one-layer
experiments evaluate lengths $\{12,16,20,24,32,40,48\}$, with accuracy over
$4096$ sequences and attention probes averaged over $4096$ sequences; peak
height is the maximum over positions of the mean attention. The two-layer
model is evaluated at $\{12,16,24,32,48\}$, with accuracy over $512$
sequences and per-head attention profiles over $256$. We use $8$ seeds for
the phenomenon and $5$ for the validation study; sweep rows use one seed.

\paragraph{Metrics.} (i) accuracy; (ii) selected peak offset (argmax of mean
attention relative to the query); (iii) carrier fraction
$\|\mathrm{mean}\|^2/\E\|\mathrm{proj}\|^2$ of projected keys and queries;
(iv) carrier-only profile $\phi_c(\delta)=\mQ^\top R_\delta\mK$; (v) decay
fit, the linear least squares of $1/a_\star$ on $L-1$.

\paragraph{Compute and code.} All experiments run on CPU; no GPU is required.
Determinism is controlled by a
single seed read immediately before each model is built. The released code is
one script with one subcommand per experiment. The main-paper figures are
produced by \texttt{multiseed} (Fig.~1), \texttt{shift} (Fig.~2),
\texttt{prop1} (Fig.~3), \texttt{mechanics} (Fig.~4), \texttt{variants}
(Fig.~5), and \texttt{transformer} (Fig.~6); \texttt{linear} produces
Table~\ref{tab:linear} of this appendix, and \texttt{mechanics},
\texttt{frequency}, and \texttt{variants} print the remaining sweep tables.

\end{document}